\documentclass{article}

\usepackage[utf8]{inputenc}
\usepackage{amsthm}
\usepackage{amsmath}
\usepackage{amsfonts}
\usepackage{graphicx}
\usepackage{fullpage}

\newcommand{\cc}{\mathcal{C}}
\newcommand{\E}{\mathbb{E}}
\newcommand{\vv}{\mathbf{v}}
\newcommand{\PoF}{\mathrm{PoF}}

\newtheorem{lemma}{Lemma}
\newtheorem{theorem}{Theorem}

\newtheorem{corollary}{Corollary}

\title{Fair Resource Allocation for Demands \\ with Sharp Lower Tail Inequalities}
\author{
  Vacharapat Mettanant\thanks{Email: vacharapat@eng.src.ku.ac.th. Department of Computer Engineering, Kasetsart University, Sriracha Campus, Chonburi, Thailand. Supported by Faculty of Engineering at Sriracha Graduate Scholarship, Kasetsart University.}
  \and
  Jittat Fakcharoenphol\thanks{Email: jittat@gmail.com. Department of Computer Engineering, Kasetsart University, Bangkok, Thailand. Supported by the Thailand Research Fund, Grant RSA-6180074.}
}

\begin{document}

\maketitle

\begin{abstract}
We consider a fairness problem in resource allocation where multiple
groups demand resources from a common source with the total fixed
amount.  The general model was introduced by Elzayn {\em et
  al.}~[FAT*'19].  We follow Donahue and Kleinberg~[FAT*'20] who
considered the case when the demand distribution is known.  We show
that for many common demand distributions that satisfy sharp lower
tail inequalities, a natural allocation that provides resources
proportional to each group's average demand performs very well.  More
specifically, this natural allocation is approximately fair and
efficient (i.e., it provides near maximum utilization).  We also show
that, when small amount of unfairness is allowed, the Price of
Fairness (PoF), in this case, is close to 1.
\end{abstract}

\section{Introduction}

Resource allocation has been a central problem in computer science and operation research~\cite{gross1956class,katoh1979polynomial,SHI2015137}.  Typically, to distribute resources well, there are many requirements to be considered. One of the most fundamental and important requirements is fairness~\cite{demers1989analysis,procaccia2013cake,eubanks2018automating}.  When fairness is a factor, in a pioneering work, Elzayn {\em et al.}~\cite{ElzaynJJKNRS19} proposed a setting where $N$ groups of people would like to obtain shared common resources, with limited amount $R$.  There is an unknown distribution for the number of candidates in each group in need of the resource.
They would like to allocate the resources so that the possibility for anyone in any group to access the resource is relatively equal, i.e., access to the distributed resource is fair.  In their setting, the distributions are unknown and they would like to learn how to allocate fairly and efficiently.  At each step, their learning algorithm provides an allocation and later receives feedback, for a particular group, on the number of candidates who received the resource.  They also showed that when the unknown distributions are Poisson or ``single-parameter Lipschitz-continuous distributions'', their learning algorithm, based on MLE, after a logarithmic number of rounds, outputs an approximately fair allocation with an almost maximum utility.  As a subroutine to their learning algorithm, they presented an algorithm for computing an optimal approximately fair allocation, assuming that candidate distributions are known.

Leaving out the learning aspect of the problem, Donahue and Kleinberg~\cite{DonalueKleinberg} considered the settings where the candidate distributions are already known and focused mostly on the trade-offs between fairness and utilization under different probability distributions, and under different allocation versions, e.g., integral and fractional allocations.  They showed many interesting results.  When the fairness is relaxed to $\alpha$-fair, they gave an upper bound on the Price of Fairness to $1/\alpha$ under fractional allocation.  They proved that when the family of distributions contains distribution that can be scaled to one another, e.g., exponential and Weibull distributions, there is no gap in fairness and utilization, i.e., PoF is 1.  They also established the bound on the Price of Fairness for Power Law distributions.

This paper follows the approach by Donahue and Kleinberg~\cite{DonalueKleinberg}.
We consider fractional resource allocation, i.e., we allow allocations where resources are distributed fractionally (or, similarly, probabilistically).
We show that when the candidate distribution $\cc_i$ for each group satisfies lower deviation tail bound, the natural way to allocate resource based on each group's mean provides both fairness and good utilization.  More specifically, when the total amount of resource is $R$, the amount of resource allocated to group $i$ is
\[
R\cdot\frac{\mu_i}{\sum_j\mu_j},
\]
where, for each group $i$, $\mu_i$ is the expected number of candidates belonging to the group.  We refer to this allocation as the {\em mean-weighted allocation}.  In contrast to Donahue and Kleinberg's results~\cite{DonalueKleinberg} that provided many examples of distributions arising from modern applications such as the Power Law distributions where the fairness-utilization gap is significant, our work shows that for many classic distributions, the natural allocation works just fine.  More over, our proofs are mostly elementary.

We would like to point out that our work is also very closely related to the results presented in Elzayn {\em et al.}~\cite{ElzaynJJKNRS19}.  On the surface, what we show here seems to be implicit in or be ``part'' of their learning algorithms that outputs approximately fair allocation with almost maximum utilization for Poisson and other distributions.  However, we note that for distributions satisfying our assumption we do not need to compute the allocations, we can just explicitly use the mean-weighted allocation.  Our fairness and utilization analysis is based on this natural allocation.  We believe that, as in the work of Donahue and Kleinberg~\cite{DonalueKleinberg}, our work simplifies the analysis and essentially shed some lights on the trade-off between the fairness and utilization for this problem.

In the next section, we review formal definitions and results of Donahue and Kleinberg~\cite{DonalueKleinberg}.  Section~\ref{sect:examples} demonstrates our intuition on why mean-weighted allocation works for distributions with mean concentration.  We specify the tail assumption in Section~\ref{sect:general-fairness} and show the fairness and utilization analysis.  Section~\ref{sect:specific-dist} provides examples on many common distributions satisfying the assumption in Section~\ref{sect:general-fairness}.

\section{Problem definitions and reviews of Donahue and Kleinberg's results}
\label{sect:def}

We follow a two-stage probabilistic model of Elzayn et al.~\cite{ElzaynJJKNRS19}, and  Donahue and Kleinberg~\cite{DonalueKleinberg}.

There are $K$ groups.  Each group $i$ has a distribution $\cc_i$ over the number of candidates $C_i$ in need of the resource.  We assume that $\E_{C_i\sim\cc_i}[C_i] > 0$ and all $C_i$'s are independent.
When the context is clear, we use $\E[C_i]$ instead of $\E_{C_i\sim\cc_i}[C_i]$ for simplicity.
We let $f_i$ be the probability density function and $F_i$ be the cumulative distribution function for $C_i$.

We have $R$ units of resource that can be distributed for these $K$ groups.  We assume that the resource is discrete; therefore, each unit of resource can be allocated to one and only one candidate.

We would like to find allocation $v_i$ of resource for each group $i$ such that $\sum v_i= R$ (i.e., we are required to allocate {\em all} the resource). 
When $v_i$ units of resource is allocated, we assume that each candidate of group $i$ has the same opportunity to receive the resource. Therefore, the probability of receiving the resource for each candidate is $\min(v_i/C_i, 1)$.
Let vector $\vv=[v_1,v_2,\ldots,v_K]$.

There are two (somewhat) competing goals.  The {\em utilization} of $\vv$ is defined as
\[
U(\vv,\{\cc_i\}):=\sum_{i=1}^K \E_{C_i\sim\cc_i}[\min(C_i, v_i)].
\]

Let $q(v,\cc)$ be the {\em availability} of the resource for a group with distribution $\cc$ when $v$ units of resource is allocated, defined as the opportunity of a candidate receiving the resource. Formally, if $x$ is a member of the group, the availability is
\[
q(v,\cc):=\Pr[x \text{ receives the resource }| x \text{ is a candidate}].
\]
In the paper of Donahue and Kleinberg~\cite{DonalueKleinberg},
they showed that 
\[
q(v,\cc) = \frac{\E_{C\sim\cc} [\min(C,v)]}{\E_{C\sim\cc}[C]}.
\]

Inspired from \emph{equality of opportunity} proposed by Hardt et al.~\cite{hardt2016equality},
we define the {\em fairness} of $\vv$ to be the maximum difference of the availability, i.e., the fairness of $\vv$ is
\[
Q(\vv,\{\cc_i\}):=\max_{i,j} |q(v_i,\cc_i) - q(v_j,\cc_j)|.
\]
If the fairness of $\vv$ is less than or equal to $\alpha$, we say that the allocation $\vv$ is \emph{$\alpha$-fair}.

Since there are two objectives, one approach is to guarantee a certain fairness with parameter $\alpha$, i.e., we would like to find an allocation $\vv$ (with $\sum_iv_i=R$) such that $Q(\vv,\{\cc_i\})\leq\alpha$ that maximizes the utilization $U(\vv,\{\cc_i\})$.  This motivates the notion of Price of Fairness (PoF), defined to be
\[
\PoF(\alpha):=
\frac{
\max_{\vv:\sum_iv_i=R} 
U(\vv,\{\cc_i\})
}{
\max_{\vv:\sum_iv_i=R} 
\left(
U(\vv,\{\cc_i\}) \ \ \mbox{s.t.} \ \ Q(\vv,\{\cc_i\})\leq\alpha \ 
\right)
}.
\]

Donahue and Kleinberg~\cite{DonalueKleinberg} consider two versions of the allocations: one where the allocations $v_i$ must be integer and one where $v_i$ can be fractional.  For integer allocation, they showed that PoF is unbounded.
When fractional or probabilistic allocations are allowed, they showed that PoF is bounded by $1/\alpha$.
Moreover, they showed, in the next theorem, that PoF is 1 for candidate distributions satisfying some condition.

\begin{theorem}[Theorem 2 from~\cite{DonalueKleinberg}]
Consider candidate distributions with $F_i(0)=0$ and $f_i(v)>0$, for $v\geq 0$.  Suppose the set of candidate distributions $\{\cc_i\}$ has the following property:
\[
F_i(v) = F_j\left(v\cdot\frac{\E[C_j]}{\E[C_i]}\right),
\]
for $v\geq 0$, for all $i,j$.  Then, under the fractional allocation of resources, the max-utilization allocation is 
$0$-fair.
\label{thm:DK-pof1}
\end{theorem}

\section{Illustrative examples}
\label{sect:examples}

To see how availability and utilization change with various allocation levels, it is useful to start with an easy case with constant candidate distribution.  For simplicity assume that the number of candidates is scaled down to be exactly 1, so that the availability and utilization are equal.  See Figure~\ref{fig:ex-av-demand}~(left).  The figure also shows the accumulative density function $F$; note that in this case it is a step function that changes from 0 to 1 at the mean $\mu$.  As the plot shows, the availability keeps increasing up to the point when the resource is enough for all candidates.

\begin{figure}
    \centering
    \includegraphics[width=0.45\textwidth]{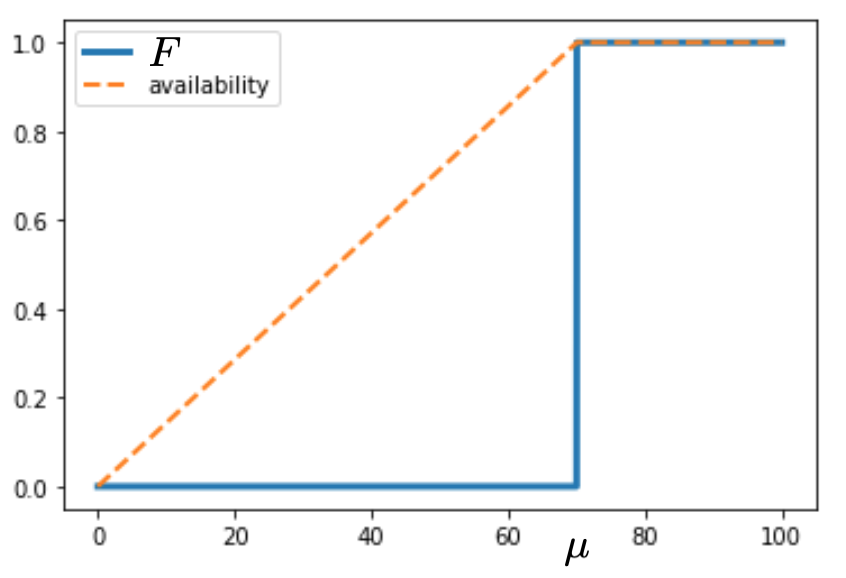}
    \includegraphics[width=0.45\textwidth]{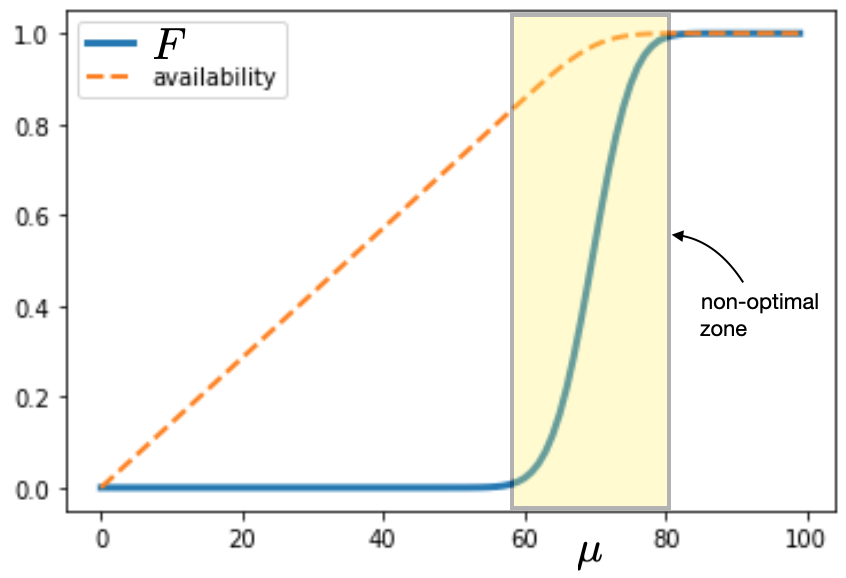}
    \caption{Availability for constant demand (left) and for normally-distributed demand (right)}
    \label{fig:ex-av-demand}
\end{figure}

Note that this case falls into the case of Theorem~\ref{thm:DK-pof1} by Donahue and Kleinberg, and we know that PoF is 1.  However, it serves as an introduction to our approach to prove that directly here.
For this case, we have a simple way to allocate all $R$ units of resource showing that PoF is 1. We allocate \[
v_i=R\cdot\frac{\mu_i}{\sum_{j}\mu_j}
\]
units to each group $i$.

\begin{lemma}
The allocation gives $PoF=1$ for constant candidates.
\end{lemma}
\begin{proof}
If $R\geq\sum_i\mu_i$, we allocate $v_i\geq\mu_i$ to each group.
The utilization will be $\sum_i\mu_i$, which is maximum.
The availability of each group is
\[
\frac{\min(\mu_i,v_i)}{\mu_i}=\frac{\mu_i}{\mu_i}=1.
\]
Since the availability of all groups are equal, this allocation is $0$-fair.

If $R<\sum_i\mu_i$, we have $v_i<\mu_i$. 
The allocation gives us $R$ utilization which also maximum.
The availability of each group is
\[
\frac{\min(\mu_i,v_i)}{\mu_i}=\frac{v_i}{\mu_i}=\frac{R}{\sum_j\mu_j}.
\]
We can see that the availability of all groups are equal as well.
Hence the allocation is $0$-fair.

In both case, the allocation is $0$-fair and gives us maximum utilization. Therefore, the PoF is 1.
\end{proof}

When dealing with non-constant demand distribution highly concentrated around its mean, we see a similar picture (See Figure~\ref{fig:ex-av-demand}~(right), for the case with normally distributed demands).  When the level of allocated resource is far from the mean $\mu$, the utilization and availability behave roughly as in the previous case.  Things get more interesting around the mean (highlighted in yellow in the figure).  If we keep distributing resource proportionally to each group's mean, we might observe the price of fairness here.   Intuitively, if the range is small, we would expect small penalty.  This is what we shall prove in Section~\ref{sect:general-fairness}.

Notably we do not need that the distribution symmetrically concentrates around its mean, we only need that the lower tail is very small.  To see that this lower concentration is crucial when using mean-weighted allocation, we provide another example where the random variable $C$ for the number candidates is defined to be such that $\Pr[C = 0] = (k-1)/k$ and $\Pr[C=k] = 1/k$.  Note that $\E[C]=1$, but the probability that $C$ is less than its expectation is very large, i.e., $\Pr[C<\E[C]] = 1-1/k$.  In this case, allocating resource $v\leq k$ to the group only yields the availability and utilization of $v/k$.

Another example is the exponential distribution considered by Donahue and Kleinberg, who showed that PoF is always 1.  In stark contrast, our approach cannot show any good bounds for this case. 
\section{General assumptions and fairness analysis}
\label{sect:general-fairness}

In this section, we provide analysis of availability, utilization and fairness for classes of candidate distributions satisfying certain concentration property.  We show in Section~\ref{sect:specific-dist} that many common distributions satisfy this condition using well-known concentration inequalities (see, e.g., a survey by Boucheron, Lugosi, and Bousquet~\cite{BoucheronLB2004-concen}).

We say that a random variable $X$ satisfies an {\em $(\epsilon,\delta)$-lower deviation inequality} 
for $0<\epsilon,\delta<1$ if
\[
\Pr[X\leq (1-\epsilon)\E[X]] \leq \delta.
\]
We say that a distribution satisfies an {\em $(\epsilon,\delta)$-lower deviation inequality} if a random variable from that distribution is with $(\epsilon,\delta)$-lower deviation inequality.  Before we continue, we note that typically the parameter $\epsilon$ is usually a small constant (say 1\%, or 10\%), and $\delta$ is a very small number, usually polynomially small.

In what follows, we assume that distributions $\{\cc_i\}$ satisfies an $(\epsilon,\delta)$-lower deviation inequality. 
Also, let $\mu_i=\E[C_i]$.  
Let $Z=\sum_{i=1}^K \mu_i$ be the total expected number of candidates over all groups.
We will use a mean-weighted allocation based on groups' mean, i.e., we let
\[
v_i=R\cdot\frac{\mu_i}{Z},
\]
for $1\leq i\leq K$.
We would show that this allocation is very fair (the fairness value is closed to 0) and gives almost optimal utilization.  
We shall use this to prove the bound on the price of fairness (PoF).

\subsection{Fairness}
We analyze the fairness in two regions based on the total resources $R$ and $Z$: 
\begin{enumerate}
    \item when $R\leq(1-\epsilon)Z$, and
    \item when $R\geq(1-\epsilon)Z$.
\end{enumerate}

\subsubsection{When $R\leq(1-\epsilon)Z$}
In this case, our allocation set
\[
v_i = R\cdot\frac{\mu_i}{Z}\leq (1-\epsilon)\mu_i.
\]

We will prove the upper bound and the lower bound on the expected availability in this case.

\begin{lemma}
\label{lemma:avai_bound_case1}
The allocation ensures
\[
\frac{v_i}{\mu_i}(1 - \delta) \leq q(v_i,\cc_i)\leq\frac{v_i}{\mu_i}\leq 1-\epsilon.
\]
\end{lemma}
\begin{proof}
First consider the upper bound. 
For each group $i$, let $U_i=\min(C_i, v_i)$ represents the utilization of the group.
Since $\E[U_i]\leq v_i$, the availability of group $i$ can be bounded by
\[
q(v_i,\cc_i)=\frac{\E[U_i]}{\mu_i}\leq\frac{v_i}{\mu_i}\leq 1-\epsilon.
\]

To show the lower bound, recall that
\[
\begin{split}
\E[U_i] &=\E[U_i|C_i< v_i]\Pr[C_i< v_i] + \E[U_i|C_i\geq v_i]\Pr[C_i\geq v_i]\\
&\geq \E[U_i|C_i\geq v_i]\Pr[C_i\geq v_i].
\end{split}
\]
Given that the number of candidates $C_i\geq v_i$, we get $U_i=\min(C_i,v_i)=v_i$.
Moreover, since $C_i$ satisfies $(\epsilon,\delta)$-lower deviation inequality, we know that 
\[
\Pr[C_i\geq v_i]\geq \Pr[C_i\geq (1-\epsilon)\mu_i]\geq 1-\delta.
\]
Using these facts, we get
\[
\E[U_i]\geq v_i(1-\delta)
\]
and the availability $q(v_i,\cc_i)$ of group $i$ can be bounded by
\[
q(v_i,\cc_i)=\frac{\E[U_i]}{\mu_i}\geq \frac{v_i}{\mu_i}(1-\delta),
\]
as required.
\end{proof}

\begin{lemma}
\label{lemma:fairness_low_v}
When $R\leq (1-\epsilon)Z$, the mean-weighted allocation gives
\[
Q(\vv,\{\cc_i\})\leq (1-\epsilon)\delta
=\delta - \epsilon\delta.
\]
\end{lemma}
\begin{proof}
From Lemma~\ref{lemma:avai_bound_case1}, we know that for each group $i$,
\[
\frac{v_i}{\mu_i}(1-\delta)\leq q(v_i,\cc_i)\leq \frac{v_i}{\mu_i}.
\]
Hence the fairness is bounded by
\[
\begin{split}
Q(\vv, \{\cc_i\}) & \leq \max_i\frac{v_i}{\mu_i} - \min_j\frac{v_j}{\mu_j}(1-\delta).
\end{split}
\]
However, by the definition of $v_i$, we know that the ratio $v_i/\mu_i=R/Z$ for all $i$ and does not depend on groups. So,
\[
\begin{split}
Q(\vv, \{\cc_i\})&\leq \frac{R}{Z}-\frac{R}{Z}(1-\delta)\\
&= \frac{R}{Z}\delta\\
&\leq (1-\epsilon)\delta.
\end{split}
\]
\end{proof}

\subsubsection{When $R\geq(1-\epsilon)Z$}

When $R\geq (1-\epsilon)Z$, our allocation will set
\[
v_i=R\cdot\frac{\mu_i}{Z}\geq(1-\epsilon)\mu_i.
\]

\begin{lemma}
In this case,
\[
(1-\epsilon)(1-\delta)
\leq
q(v_i,\cc_i)
\leq 
1.
\]
\label{lemma:avai_bound_case2}
\end{lemma}
\begin{proof}
Consider the lower bound. Since $v_i\geq (1-\epsilon)\mu_i$, we get that
\[
\E[U_i]=\E[\min(C_i,v_i)]\geq\E[\min(C_i,(1-\epsilon)\mu_i)].
\]
and the availability of each group $i$ can be bounded by
\[
q(v_i,C_i)\geq\frac{\E[\min(C_i,(1-\epsilon)\mu_i)]}{\mu_i}.
\]

As in the proof of Lemma~\ref{lemma:avai_bound_case1}, recall that
\[
\begin{split}
& \E[\min(C_i,(1-\epsilon)\mu_i)] \\
&\geq \E[\min(C_i,(1-\epsilon)\mu_i|C_i\geq(1-\epsilon)\mu_i]\Pr[C_i\geq(1-\epsilon)\mu_i]\\
&=(1-\epsilon)\mu_i\Pr[C_i\geq(1-\epsilon)\mu_i]\\
&\geq (1-\epsilon)(1-\delta)\mu_i
\end{split}
\]
since $C_i$ satisfies $(\epsilon,\delta)$-lower deviation inequality.
Therefore, 
\[
q(v_i, C_i)\geq(1-\epsilon)(1-\delta).
\]

For the upper bound, note that from
\[
\E[U_i]=\E[\min(C_i,v_i)]\leq\E[C_i]=\mu_i,
\]
we know that $q(v_i,C_i)\leq 1$.
\end{proof}

Therefore, we have the following corollary.

\begin{corollary}
When $R\geq (1-\epsilon)Z$, we have that
\[
Q(\vv,\{\cc_i\})\leq 1-(1-\epsilon)(1-\delta)= \epsilon+\delta-\epsilon\delta.
\]
\end{corollary}

From both cases of $R$, we can conclude as followed.
\begin{lemma}
When the distributions $\{\cc_i\}$ have ($\epsilon,\delta$)-lower deviation inequality, the mean-weighted allocation gives the fairness within $\epsilon + \delta - \epsilon\delta$.
\label{lemma:fairness}
\end{lemma}

\subsection{Utilization}
This section shows that the mean-weighted allocation also gives a very good utilization bound.
Before we start, recall that the maximum expectation of total utilization is at most $\min(R,Z)$.  We first consider the case when $R\leq (1-\epsilon)Z$.

\begin{lemma}
If $R\leq(1-\epsilon)Z$, the utilization is at least $(1-\delta)R$.
\end{lemma}
\begin{proof}
Recall that the utilization is defined as
\[
U(\vv,\{\cc_i\})=\sum_{i=1}^K\E[U_i].
\]
From our proof of Lemma~\ref{lemma:avai_bound_case1},
we have $\E[U_i]\geq (1-\delta)v_i$ for each $i$.
Therefore, the utilization is at least
\[
U(\vv,\{\cc_i\})\geq\sum_{i=1}^K (1-\delta)v_i = (1-\delta)R.
\]
\end{proof}

On the other hand, when $R\geq (1-\epsilon)Z$, 
we show that the utilization is at least $(1-\epsilon-\delta)Z$.
\begin{lemma}
If $R\geq (1-\epsilon)Z$, the utilization is at least 
\[
(1-\epsilon-\delta)Z.
\]
\end{lemma}
\begin{proof}
From the proof of Lemma~\ref{lemma:avai_bound_case2}, we have
\[
\E[U_i]\geq (1-\epsilon)\mu_i\Pr[C_i\geq(1-\epsilon)\mu_i]
\geq (1-\epsilon)(1-\delta)\mu_i
\]
in this case.
Therefore, the utilization is
\[
\begin{split}
U(\vv,\{\cc_i\})=\sum_{i=1}^K\E[U_i]
&\geq \sum_{i=1}^K (1-\epsilon)(1-\delta)\mu_i\\
&=(1-\epsilon)(1-\delta)Z\\
&\geq (1-\epsilon-\delta)Z.
\end{split}
\]
\end{proof}

These two lemmas imply the following key lemma.

\begin{lemma}
When the candidate distributions satisfy the $(\epsilon,\delta)$-lower deviation inequality, 
the utilization for the mean-weighted allocation is at least
\[
\min(1-\delta, 1-\epsilon-\delta) 
=
1-\epsilon - \delta 
\]
of the maximum utilization.  
\label{lemma:utilization}
\end{lemma}

\subsection{The bound on PoF}

Assume that $\epsilon+\delta< 1$.
We use the bounds from Lemma~\ref{lemma:fairness} and Lemma~\ref{lemma:utilization} to show that when
\[
\epsilon+\delta\leq\alpha<1,
\]
the Price of Fairness is at most
\[
\frac{1}{1-\epsilon - \delta}
\leq \frac{1}{1-\alpha}.
\]
Moreover, if $\epsilon+\delta\leq 1/2$, we have
\[
\frac{1}{1-\epsilon-\delta}=1+\frac{\epsilon+\delta}{1-\epsilon-\delta}\leq 1+2(\epsilon+\delta)\leq 1+2\alpha.
\]
Thus, we have the following main theorem.

\begin{theorem}
\label{thm:pof}
If candidate distributions $\{\cc_i\}$ satisfy the $(\epsilon,\delta)$-lower deviation inequality for $\epsilon,\delta$ such that $\epsilon+\delta<1$, the Price of Fairness (PoF) when $\alpha\geq\epsilon+\delta$ is at most $1/(1-\alpha)$. In addition, if $\epsilon+\delta\leq 1/2$ the PoF is at most $1+2\alpha$.
\end{theorem}
\section{Results for specific distributions}
\label{sect:specific-dist}

In this section, we show that many common distributions, for demand modeling, satisfies the $(\epsilon,\delta)$-lower deviation inequality.  We only provide a few examples.

\subsection{Binomial distribution}

Assume that there are $n_i$ people in group $i$, and independently each person in group $i$ would be a candidate with probability $p_i$.  The number of candidates in group $i$, $C_i$, is a binomial random variable with parameter $n_i$ and $p_i$.  We have, for an integer $x$ such that $0\leq x\leq n_i$,
\[
\Pr[C_i = x] = \binom{n_i}{x}p^x (1-p)^{n_i-x},
\]
with $\mu_i=n_ip_i$.  For this type of random variables, we can apply the Chernoff bound to get that
\[
\Pr[C_i\leq (1-\epsilon)\mu_i]\leq e^{-\mu_i\epsilon^2/2}.
\]
Note that the term $e^{-\mu_i\epsilon^2/2}$ specifies the parameter $\delta$ and is dependent on $\mu_i$.  Thus, if we take $\epsilon, \delta$, and $p_i$ to be fixed, we have the following lemma.
\begin{lemma}
Assume that the candidate distributions are all binomial. For any $\epsilon,\delta$ such that $\epsilon+\delta\leq 1/2$, and for any $p_i$, The number of candidates $C_i$ satisfies the $(\epsilon,\delta)$-lower deviation inequality when
\[
n_i\geq \frac{2}{\epsilon^2p_i}\ln\frac{1}{\delta}.
\]
\label{lemma:binom}
\end{lemma}

\begin{proof}
When $n_i\geq \frac{2}{\epsilon^2p_i}\ln\frac{1}{\delta}$, we have
$e^{-\mu_i\epsilon^2/2}\leq\delta$. This fact implies that $\Pr[C_i\leq(1-\epsilon)\mu_i]\leq\delta$, which is the definition of 
$(\epsilon,\delta)$-lower deviation inequality.
\end{proof}

\subsection{Normal distribution}

Normal distribution or Gaussian distribution is a continuous distribution whose random variable $C$ with parameter mean $\mu$ and standard deviation $\sigma$ has the density probability distribution $f$ defined as
\[
f(x) = \frac{1}{\sigma\sqrt{2\pi}}e^{-\frac{1}{2}\left(\frac{x-\mu}{\sigma}\right)^2}.
\]
Normal distributions are catch-all distributions, used in numerous modelings calculations when the distributions is not clear or unknown.

In the context of this problem, the number of candidates $C_i$ for each group $i$ is a normal random variable with mean $\mu_i$ and standard deviation $\sigma_i$. Using the Chernoff bound, we have that
\[
\Pr[C_i\leq (1-\epsilon)\mu_i]\leq e^{-\frac{\epsilon^2\mu_i^2}{2\sigma_i^2}}.
\]
Again, with the same argument as in Lemma~\ref{lemma:binom}, this implies that Normal random variable $C_i$ satisfies the $(\epsilon,\delta)$-lower deviation inequality when $\delta\geq e^{-\frac{\epsilon^2\mu_i^2}{2\sigma_i^2}}$,
which implies
\[
\mu_i\geq\sqrt{\frac{2\sigma_i^2}{\epsilon^2}\ln\frac{1}{\delta}}.
\]

\subsection{Poisson distribution}

Poisson distribution is a discrete distribution typically used to express the number of events occurring in the particular time period (usually for rare events).  A Poisson random variable $C$ with parameter $\lambda$ satisfies
\[
\Pr[C=x] = \frac{\lambda^x e^{-\lambda}}{x!},
\]
for integer $x=0,1,\ldots$.  The expectation $\E[C]$ is $\lambda$.  It can be viewed as the limit of the binomial distribution (i.e., fixing $\lambda=n p$, and take $n\rightarrow\infty$).  To quote Feller~\cite{feller1}, examples of observations fitting the Poisson distribution are radioactive disintegrations, flying-bomb hits on London, chromosome interchanges in cells, connections to wrong number, and bacteria and blood counts.

In the context of this problem, we consider the situation when the number of candidates $C_i$ for each group $i$ is a Poisson random variable with parameter $\lambda_i$.  It is folklore that Poisson random variables have sub-exponential concentration bounds.  The following is from Canonne's note~\cite{Canonne-poisson}:
\[
\Pr[C_i<(1-\epsilon)\lambda_i] \leq e^{-\frac{\epsilon^2\lambda_i}{2}h(-\epsilon)},
\]
where $h(x)=2\frac{(1+x)\ln(1+x)-x}{x^2}$.  Thus, when each $\lambda_i$ is large enough, i.e., when
\[
\lambda_i\geq\frac{2}{\epsilon^2h(-\epsilon)}\ln\frac{1}{\delta},
\]
we obtain our required assumption.

When each $C_i$ is a random variable of one of these three specific distributions, we can see that if the mean is large enough, 
$C_i$ satisfies the $(\epsilon,\delta)$-lower deviation inequality for any $\epsilon$ and $\delta$.
Thus, given $\alpha>0$, we can choose $\epsilon$ and $\delta$ such that $\epsilon+\delta\leq\min(\alpha, 1/2)$. Then, combined with Lemma~\ref{lemma:fairness}, Lemma~\ref{lemma:utilization}, and Theorem~\ref{thm:pof}, we can conclude as followed.

\begin{theorem}
Assume that the distribution of each $C_i$ is binomial, normal, or Poisson. Given that all the mean $\E[C_i]$ are large enough,
for any $\alpha\in(0,1)$, the mean-weighted allocation is $\alpha$-fair and gives us at least $(1-\alpha)$ of the maximum utilization. The PoF of this case is at most $1+2\alpha$.
\label{thm:binom}
\end{theorem}

\subsection{Other examples}

There are many other experiments that result in random variables satisfying the required $(\epsilon,\delta)$-lower deviation inequality, e.g., sub-Gaussian random variables and those random variables which are applicable to strong classic tail inequalities, such as the Chernoff's bound, Hoeffding's bound, Azuma's inequality, and McDiarmid's inequality.  For examples, the number of empty bins in a balls-and-bins experiment.  See more from classic probability textbooks, e.g.,\cite{motwani95, Mitzenmacher}, or surveys~\cite{BoucheronLB2004-concen}.  

\bibliography{fair}

\begin{thebibliography}{10}

\bibitem{BoucheronLB2004-concen}
S.~Boucheron, G.~Lugosi, and O.~Bousquet.
\newblock {\em Concentration Inequalities}, pages 208--240.
\newblock Springer Berlin Heidelberg, Berlin, Heidelberg, 2004.

\bibitem{Canonne-poisson}
C.~Canonne.
\newblock A short note on poisson tail bounds.
\newblock Available at URL:
  http://www.cs.columbia.edu/~ccanonne/files/misc/2017-poissonconcentration.pdf
  (2020/10/7), 2017.

\bibitem{demers1989analysis}
A.~Demers, S.~Keshav, and S.~Shenker.
\newblock Analysis and simulation of a fair queueing algorithm.
\newblock {\em ACM SIGCOMM Computer Communication Review}, 19(4):1--12, 1989.

\bibitem{DonalueKleinberg}
K.~Donahue and J.~Kleinberg.
\newblock Fairness and utilization in allocating resources with uncertain
  demand.
\newblock In {\em Proceedings of the 2020 Conference on Fairness,
  Accountability, and Transparency}, FAT* ’20, page 658–668, New York, NY,
  USA, 2020. Association for Computing Machinery.

\bibitem{ElzaynJJKNRS19}
H.~Elzayn, S.~Jabbari, C.~Jung, M.~J. Kearns, S.~Neel, A.~Roth, and
  Z.~Schutzman.
\newblock Fair algorithms for learning in allocation problems.
\newblock In {\em Proceedings of the Conference on Fairness, Accountability,
  and Transparency, FAT* 2019, Atlanta, GA, USA, January 29-31, 2019}, pages
  170--179. {ACM}, 2019.

\bibitem{eubanks2018automating}
V.~Eubanks.
\newblock {\em Automating inequality: How high-tech tools profile, police, and
  punish the poor}.
\newblock St. Martin's Press, 2018.

\bibitem{feller1}
W.~Feller.
\newblock {\em An Introduction to Probability Theory and Its Applications},
  volume~1.
\newblock Wiley, January 1968.

\bibitem{gross1956class}
O.~Gross.
\newblock A class of discrete-type minimization problems.
\newblock Technical report, RAND CORP SANTA MONICA CA, 1956.

\bibitem{hardt2016equality}
M.~Hardt, E.~Price, and N.~Srebro.
\newblock Equality of opportunity in supervised learning.
\newblock In {\em Advances in neural information processing systems}, pages
  3315--3323, 2016.

\bibitem{katoh1979polynomial}
N.~Katoh, T.~Ibaraki, and H.~Mine.
\newblock A polynomial time algorithm for the resource allocation problem with
  a convex objective function.
\newblock {\em Journal of the Operational Research Society}, 30(5):449--455,
  1979.

\bibitem{Mitzenmacher}
M.~Mitzenmacher and E.~Upfal.
\newblock {\em Probability and Computing: Randomized Algorithms and
  Probabilistic Analysis.}
\newblock Cambridge University Press, 2005.

\bibitem{motwani95}
R.~Motwani and P.~Raghavan.
\newblock {\em Randomized Algorithms}.
\newblock Cambridge University Press, Cambridge; NY, 1995.

\bibitem{procaccia2013cake}
A.~D. Procaccia.
\newblock Cake cutting: not just child's play.
\newblock {\em Communications of the ACM}, 56(7):78--87, 2013.

\bibitem{SHI2015137}
C.~Shi, H.~Zhang, and C.~Qin.
\newblock A faster algorithm for the resource allocation problem with convex
  cost functions.
\newblock {\em Journal of Discrete Algorithms}, 34:137 -- 146, 2015.

\end{thebibliography}
\bibliographystyle{abbrv}

\end{document}